\documentclass[letterpaper, 10 pt, conference]{ieeeconf}  

\IEEEoverridecommandlockouts 

\usepackage{graphicx}
\usepackage{subcaption}
\usepackage{etoolbox}
\usepackage{scalerel}

\usepackage{siunitx}
\usepackage{pifont}
\usepackage{graphicx}
\usepackage{amsmath,amssymb,textcomp}
\usepackage[ruled]{algorithm} 
\usepackage{cite}
\usepackage{tikz}

\allowdisplaybreaks

\usepackage{booktabs}
\usepackage{url}
\usepackage{algpseudocode}
\usepackage{paralist}
\usepackage{color}
\usepackage{stmaryrd}
\usepackage{epsfig} 
\usepackage{bm} 
\usepackage{listings,lstautogobble}
\usepackage{mathtools}
\usepackage{mathrsfs}
\usepackage{xspace}
\usepackage{verbatim}
\usepackage{epstopdf}
\usepackage{listings}
\usepackage{parcolumns}
\usepackage{color}
\usepackage{mathtools}
\usepackage{mathrsfs}
\usepackage{xspace}
\usepackage{verbatim}
\usepackage[utf8]{inputenc}
\usepackage{calrsfs}
\usepackage{rotating,multirow}

\usepackage[T1]{fontenc}
\usepackage[utf8]{inputenc}

\newtheorem{definition}{Definition}

\newtheorem{theorem}{Theorem}

  \renewcommand{\footnotesize}{\fontsize{8pt}{11pt}\selectfont}
\usepackage{authblk}

\DeclareMathAlphabet{\mathcal}{OMS}{cmsy}{m}{n}

\DeclarePairedDelimiter{\abs}{\lvert}{\rvert}
\DeclarePairedDelimiter{\norm}{\lVert}{\rVert}

\usepackage{hyperref}
\hypersetup{
    colorlinks=true,
    linkcolor=black,
    urlcolor=gray,
}
\usepackage{balance}

\newtheorem{problem}{Problem}[section]

\usepackage{titlesec}
\titlespacing{\section}{0pt}{1ex}{0.8ex}
\titlespacing{\subsection}{0pt}{0.4ex}{0.4ex}
\titlespacing{\subsubsection}{0pt}{0.3ex}{0.3ex}

\newcommand{\R}{{\mathbb{R}}}

\newcommand{\argmin}{\textrm{arg}\min}

\newcommand{\zono}[1]{\langle #1 \rangle}
\newcommand{\setdef}[2][]{
	\left\{
	\ifblank{#1}{}{#1 \hspace{.1cm} \middle| \hspace{.1cm}}
	#2
	\right\}
}

\begin{document}

\title{\LARGE \bf Enhancing Data-Driven Reachability Analysis \\
using Temporal Logic Side Information}
\author{Amr Alanwar$^*$, Frank J. Jiang$^*$, Maryam Sharifi, Dimos V. Dimarogonas, and Karl H. Johansson

\thanks{$^*$ Authors are with equal contributions and ordered alphabetically.}
\thanks{This work is supported by the the Swedish Research
Council (VR), the Wallenberg AI, Autonomous Systems and Software Program (WASP) funded by the Knut and Alice Wallenberg Foundation, ERC CoG LEAFHOUND, and the Knut \& Alice Wallenberg Foundation (KAW).}
\thanks{
    All authors are with the Division of Decision and Control Systems, EECS, KTH Royal Institute of Technology, Malvinas v{\"a}g 10, 10044 Stockholm, Sweden. All authors are also affiliated with Digital Futures. {\tt\small \{alanwar, frankji, msharifi, dimos, kallej\}@kth.se}}}

\maketitle

\begin{abstract}
 
This paper presents algorithms for performing data-driven reachability analysis under temporal logic side information.
In certain scenarios, the data-driven reachable sets of a robot can be prohibitively conservative due to the inherent noise in the robot's historical measurement data. In the same scenarios, we often have side information about the robot's expected motion (e.g., limits on how much a robot can move in a one-time step) that could be useful for further specifying the reachability analysis. In this work, we show that if we can model this side information using a signal temporal logic (STL) fragment, we can constrain the data-driven reachability analysis and safely limit the conservatism of the computed reachable sets. Moreover, we provide formal guarantees that, even after incorporating side information, the computed reachable sets still properly over-approximate the robot's future states. Lastly, we empirically validate the practicality of the over-approximation by computing constrained, data-driven reachable sets for the Small-Vehicles-for-Autonomy (SVEA) hardware platform in two driving scenarios.
\end{abstract}
\section{Introduction}

Reachability analysis is an essential tool that provides a principled understanding of the dynamic capabilities of a system \cite{conf:reach2000,conf:reach1971}.
In recent years, researchers have proposed a variety of formulations in which reachability analysis  provides formal guarantees on the safety of an autonomous system (i.e., for autonomous vehicles~\cite{conf:thesisalthoff} and drones~\cite{Gillula2011}). Traditionally, 
a reachable set of states is computed based on a model of the subject system using either set-propagation techniques~\cite{conf:rigorousreachtaylor, conf:polydiscrete, conf:sparsepolyzono} or simulation-based techniques~\cite{conf:reachsim3, Lew2020, conf:reachsim2, conf:simreachmatrix}. Most techniques compute over-approximations of the robot's reachable states to ensure that the resulting reachable set can be used for providing safety guarantees. However, these traditional approaches are sensitive to model error and do not incorporate the readily available trajectory data that robots continuously produce.


Several recent works have proposed performing reachability analysis from data~\cite{Devonport2020, conf:onthefly, conf:murat_christoffel,conf:activelearning1,conf:activelearning2,Allen2014,conf:uncertain,conf:stochasticreach,conf:koopmanblack} to overcome the limitation of prior model knowledge. By performing reachability analysis directly from data, we can form a direct link between the actual, historical performance of a robot and our prediction of its reachability, removing the dependency on the accuracy of first-principles-based modeling. Moreover, in~\cite{conf:ourL4dc,conf:ourjournal}, authors provide formal guarantees on the over-approximation of a system's reachability based on data that contains noise. However, in order to provide guarantees on the over-approximation of the data-driven reachable sets, the computed sets might become prohibitively conservative when the noise becomes significant. In this work, we aim to limit this conservatism whenever we have useful side information. 


\begin{figure}[t]
    \centering
    \includegraphics[width=0.37\textwidth]{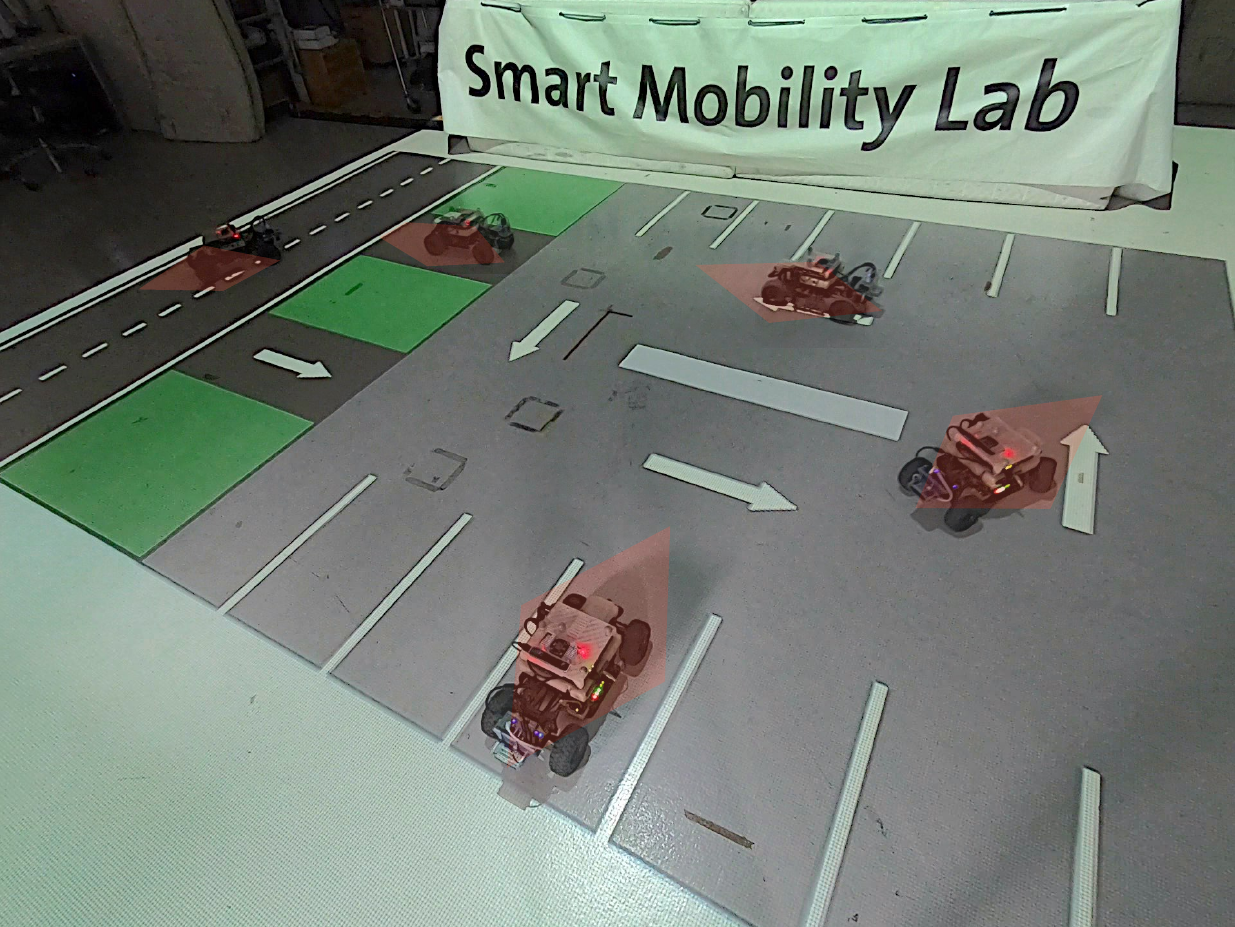}
    \caption{Snapshots from our \href{https://bit.ly/DataReachSTL}{experiment video} where an SVEA vehicle leaves a parking lot in the \href{https://www.kth.se/is/dcs/research/control-of-transport/smart-mobility-lab/smart-mobility-lab-1.441539}{Smart Mobility Lab} and its STL reachable sets are drawn in red at different time instances.} 
    \label{fig:exp}
     \vspace{-8mm}
\end{figure}

 The main contribution of this paper is an approach for performing data-driven reachability analysis under signal temporal logic (STL) side information. We choose to use STL since it can be interpreted over continuous-time signals, supports imposing strict deadlines and robust semantics \cite{fainekos2009robustness}, and allows for the formulation of complex specifications. 
To the extent of the authors' knowledge, the presented approach is novel in its use of STL formulae as side information instead of as specifications (e.g.~\cite{lindemann2020barrier, Jiang2020a}) while performing reachability analysis. 
More specifically, the contributions of this work are as follows:
(1) We provide two algorithms for performing data-driven reachability analysis under STL side information, which, in turn, reduces the conservatism of data-driven reachable sets.
(2) We provide state inclusion guarantees in reachable sets by intersecting a predicate function constructed from STL side information with either reachable zonotopes or reachable constrained zonotopes. 
(3) We validate our approach in two driving scenarios using the Small-Vehicles-for-Autonomy (SVEA) hardware platform (e.g., in Fig.~\ref{fig:exp}).

The remainder of the paper is organized as follows. In Section~\ref{sec:prelim}, we introduce preliminary material. 
In Section~\ref{sec:approach}, we present our approach to constrain the reachable sets using STL-based side information. In Section~\ref{sec:eval}, we validate the practicality of our approach using the SVEA platform. In Section~\ref{sec:con}, we conclude the paper with final remarks.

\section{Preliminaries and Problem Statement}\label{sec:prelim}
In this section, we start by describing our assumed model for the subject system. After establishing our assumed model, we overview some necessary preliminary material and end the section by detailing the problem that we solve in Section~\ref{sec:approach}.
\subsection{Model Description}
We consider a discrete-time Lipschitz nonlinear system
\begin{align}
    x(k+1) &= f(x(k),u(k))+ w(k). 
    \label{eq:sys}
\end{align}
We assume $f$ to be an unknown twice differentiable function and $w(k) \in \mathcal{Z}_w$ to be process noise bounded by the set $\mathcal{Z}_w$. 

\subsection{Reachable Set and Set Representations}
In the following definitions, we define the reachable sets and different set representations used in our approach. 
\begin{definition} \textbf{(Reachable set)}
The reachable set $\mathcal{R}_{N}$ after $N$ steps of system \eqref{eq:sys} from a set of initial states $\mathcal{X}_0 \subset \mathbb{R}^n$ and a set of possible inputs $\mathcal{U}_k \subset \mathbb{R}^m$ is
\begin{align}
    \begin{split} \label{eq:R}
        \mathcal{R}_{N} =& \big\{ x(N) \in \mathbb{R}^n \, \big| \forall k \in \{0,...,N-1\}:\nonumber\\
        &\, x(k+1) = f(x(k),u(k))+ w(k), w(k) \in \mathcal{Z}_w,   \nonumber\\
        &\,   u(k) \in \mathcal{U}_k,x(0) \in \mathcal{X}_0  \big\}.
    \end{split}
\end{align}
\end{definition}

\begin{definition} \label{df:zono} \textbf{(Zonotope \cite{conf:zono1998,Girard2005})}
Given center $c_{\mathcal{Z}} \in \mathbb{R}^n$ and $\gamma_{\mathcal{Z}} {\in} \mathbb{N}$ generator vectors in a generator matrix $G_{\mathcal{Z}}{=}\begin{bmatrix} g_{\mathcal{Z}}^{(1)}& {\dots} &g_{\mathcal{Z}}^{(\gamma_{\mathcal{Z}})}\end{bmatrix} {\in} \mathbb{R}^{n \times \gamma_{\mathcal{Z}}}$, a zonotope is defined as
\begin{equation*}
\mathcal{Z} = \Big\{ x \in \mathbb{R}^n \; \Big| \; x = c_{\mathcal{Z}} + G_{\mathcal{Z}} \beta_{\mathcal{Z}}, \, 	-1 \leq \beta_{\mathcal{Z}}^{(i)} \leq 1 \Big\} \, .
\end{equation*}
We use the shorthand notation $\mathcal{Z} = \zono{c_{\mathcal{Z}},G_{\mathcal{Z}}}$ for a zonotope. 
\end{definition}

The linear map is defined and computed as follows \cite{conf:diff-set}:
\begin{align}
L \mathcal{Z} = \{Lz | z\in\mathcal{Z}\}  = \zono{L c_{\mathcal{Z}}, L G_{\mathcal{Z}} }. \label{eq:linmap}
\end{align}
Given two zonotopes $\mathcal{Z}_1=\langle c_{\mathcal{Z}_1},G_{\mathcal{Z}_1} \rangle$ and $\mathcal{Z}_2=\langle c_{\mathcal{Z}_2},G_{\mathcal{Z}_2} \rangle$, the Minkowski sum $\mathcal{Z}_1 + \mathcal{Z}_2 = \{z_1 + z_2| z_1\in \mathcal{Z}_1, z_2 \in \mathcal{Z}_2 \}$ can be computed exactly as \cite{conf:diff-set}: 
\begin{equation}
     \mathcal{Z}_1 + \mathcal{Z}_2 = \Big\langle c_{\mathcal{Z}_1} + c_{\mathcal{Z}_2}, [G_{\mathcal{Z}_1}, G_{\mathcal{Z}_2} ]\Big\rangle.
     \label{eq:minkowski}
\end{equation}
%
The Cartesian product of two zonotopes $\mathcal{Z}_1 $ and $\mathcal{Z}_2$ is defined and computed as 
\begin{align}
\mathcal{Z}_1 \times \mathcal{Z}_2 &= \Bigg\{ \begin{bmatrix}z_1 \\ z_2\end{bmatrix} \Bigg| z_1 \in \mathcal{Z}_1, z_2 \in \mathcal{Z}_2 \Bigg\}\nonumber\\
&= \bigg\langle \begin{bmatrix} c_{\mathcal{Z}_1}  \\ c_{\mathcal{Z}_2}  \end{bmatrix}, \begin{bmatrix} G_{\mathcal{Z}_1}  & 0 \\ 0 & G_{\mathcal{Z}_2}  \end{bmatrix} \bigg\rangle.
\label{eq:cardprod}
\end{align}

The noise $w(k)$ is random but bounded by the zonotope $w(k) \in \mathcal{Z}_w=\zono{c_{\mathcal{Z}_w},G_{\mathcal{Z}_w}}$. With a minor abuse of notation, we write $\mathcal{Z} =  \text{zonotope}(\underline{l},\overline{l}) \subset \mathbb{R}^n$ to represent an interval vector $\mathcal{I}=[\underline{l},\overline{l}] \subset \mathbb{R}^n$ as a zonotope where the interval vector is defined element wise. Zonotopes have been extended in \cite{conf:const_zono} to represent polytopes by applying constraints on the factors multiplied with the generators.

\begin{definition}\label{df:contzono}
\textbf{(Constrained zonotope \cite{conf:const_zono})} An $n$-dimensional constrained zonotope is defined by
\begin{equation*}
    \mathcal{C} {=} \setdef[x\in\mathbb{R}^n]{x=c_{\mathcal{C}}+G_{\mathcal{C}} \beta_{\mathcal{C}}, A_{\mathcal{C}} \beta_{\mathcal{C}}=b_{\mathcal{C}}, -1 \leq \beta_{\mathcal{C}}^{(i)} \leq 1},
\end{equation*}
where $c_{\mathcal{C}} \in \R^n$ is the center, $G_{\mathcal{C}}$ $\in$ $\R^{n \times \gamma_{\mathcal{C}}}$ is the generator matrix and $A_{\mathcal{C}} \in $ $\R^{n_c \times \gamma_{\mathcal{C}}}$ and $b_{\mathcal{C}} \in \R^{n_c}$ constrain the factors $\beta_{\mathcal{C}}$. In short, we write $\mathcal{C}= \zono{c_{\mathcal{C}},G_{\mathcal{C}},A_{\mathcal{C}},b_{\mathcal{C}}}$.
\end{definition}


\begin{definition}\label{df:strip}
\textbf{(Strip \cite{conf:stripzono})} 
For given parameters $y_{i,k}\in \mathbb{R}^p, H_{i,k}\in \mathbb{R}^{p \times n}$ and $r_{i,k}\in \mathbb{R}^p$, the strip $\mathcal{S}_{i,k}$ of index $i$ is the set of all possible state values satisfying
\begin{equation}
    \mathcal{S}_{i,k} = \setdef[x]{\abs{ H_{i,k} x - y_{i,k}} \leq r_{i,k} }, \label{eq:strip}
\end{equation}
where $|\cdot|$ and $\leq$ are applied element wise. 
\end{definition}

\begin{definition}\label{df:nonlinstrip}
\textbf{(Nonlinear strip)} 
For given $h_{i,k}(x)\in \mathbb{R}^p$ and $r_{i,k}\in \mathbb{R}^p$ the nonlinear strip $\mathcal{N}_{i,k}$ of index $i$ is the set of all possible state values satisfying
\begin{equation}
    \mathcal{N}_{i,k} = \setdef[x]{\abs{ h_{i,k}(x)} \leq r_{i,k} }. \label{eq:nonlinstrip}
\end{equation}
\end{definition}

We denote the Moore-Penrose pseudoinverse by $\dagger$ and the Kronecker product by $\otimes$. We also the denote the $j^{\text{th}}$ column of a matrix $A$ by ${(A)}_{.,j}$. The Frobenius norm is denoted by by $\|.\|_F$. For simplicity, we consider dimension of $p=1$.

\subsection{Signal Temporal Logic}
STL is an expressive language that is able to model complex, time-varying side information. STL is based on predicates $\nu$ which are obtained by evaluation of a predicate function $\mathfrak{h}(x):\mathbb{R}^{n}\to\mathbb{R}$, where $\nu:=\top$ (True) if $\mathfrak{h}(x)\geq 0$ and $\nu:=\bot$ (False) if $\mathfrak{h}(x)< 0$ for $x\in\mathbb{R}^{n}$ \cite{maler2004monitoring}. 
In this paper, we consider side information that can be modeled with the following STL fragment:
	\begin{align}
&\bar\phi::=G_{\left[ {a,b} \right]}\phi|F_{\left[ {a,b} \right]}\phi|\phi'U_{\left[ {a,b} \right]}\phi''|\phi'\wedge\phi'',\label{2nd}
\end{align}
where $\phi,\phi', \phi''$ are STL formulas.
In addition, $U_{\left[ {a,b} \right]}$ is the until operator with $a\leq b < \infty$, and $F_{\left[ {a,b} \right]}\phi=\top U_{\left[ {a,b} \right]}\phi$ and $G_{\left[ {a,b} \right]}\phi=\neg F_{\left[ {a,b} \right]}\neg\phi$ are eventually and always operators, respectively. Let $(x,t)\models\bar\phi$ denote the satisfaction relation. A formula $\bar\phi$ is satisfiable if $\exists x:\mathbb{R}_{\geq 0}\to\mathbb{R}^n$ such that $(x,t)\models\bar\phi$.
STL semantics are defined formally as follows:
\begin{definition}\textbf{(STL semantics \cite{maler2004monitoring})} The STL semantics for a signal $x:\!\!~\mathbb{R}_{\geq 0}\to\mathbb{R}^n$ are recursively given by:
\begin{align*}
  &(x,t)\models\nu \;\;\;\;\;\;\;\;\;\;\;\;\;\Leftrightarrow \mathfrak{h}(x)\geq 0,\\
  &(x,t)\models\neg\phi\;\;\;\;\;\;\;\;\;\;\;\Leftrightarrow \neg((x,t)\models\phi),\\
  &(x,t)\models\phi'\wedge\phi''\;\;\;\;\Leftrightarrow (x,t)\models\phi'\wedge(x,t)\models\phi'',\\
    &(x,t)\models\phi' U_{\left[ {a,b} \right]}{\phi ''}\Leftrightarrow \exists t_1\in{\left[ {t+a,t+b} \right]}\; s.t. (x,t_1)\models\phi''\\&\;\;\;\;\;\;\;\;\;\;\;\;\;\;\;\;\;\;\;\;\;\;\;\;\;\;\;\;\;\;\;\;\;\wedge\forall t_2\in{\left[ {t,t_1} \right]}, (x,t_2)\models\phi',\\
    &(x,t)\models F_{\left[ {a,b} \right]}{\phi}\;\;\;\;\;\Leftrightarrow\exists t_1\in{\left[ {t+a,t+b} \right]}\; s.t. (x,t_1)\models\phi,\\
        &(x,t)\models G_{\left[ {a,b} \right]}{\phi}\;\;\;\;\;\Leftrightarrow\forall t_1\in{\left[ {t+a,t+b} \right]}\; s.t. (x,t_1)\models\phi. 
\end{align*}
\end{definition}
We omit the time to simplify the notation and write $x \models \phi$.

\subsection{Data-Driven Reachablibity Analysis}
In this section, we show how we compute data-driven reachable sets from recorded trajectories. 
Consider $K$ input-state data trajectories of length $T_j$, $j=1,\dots,K$, from system \eqref{eq:sys}, given by $\{u^{(j)}(k)\}_{k=0}^{T_j - 1}$, $\{x^{(j)}(k)\}_{k=0}^{T_j}$.
Denote the following matrices containing the set of all data sequence.
\begin{align*}
     X &{=}  \begin{bmatrix} x^{(1)}(0) \dots  x^{(1)}(T_1)  \dots  x^{(K)}(0) \dots  x^{(K)}(T_K)\end{bmatrix}, \\
     U_- &{=}  \begin{bmatrix} u^{(1)}(0)  \dots  u^{(1)}(T_1\!-\!1) \dots u^{(K)}(0)  \dots  u^{(K)}(T_K\!-\!1) \end{bmatrix}, \\
     X_+ &{=}  \begin{bmatrix} x^{(1)}(1)\dots  x^{(1)}(T_1) \dots  x^{(K)}(1)  \dots  x^{(K)}(T_K) \end{bmatrix}, \nonumber\\
     X_- &{=}  \begin{bmatrix} x^{(1)}(0) \dots  x^{(1)}(T_1\!-\!1)  \dots  x^{(K)}(0)  \dots  x^{(K)}(T_K\!-\!1) \end{bmatrix}.
 \end{align*}
The total number of data points is denoted by $T {=} \sum_{j=1}^{K} T_j$, and the set of all data by $D{=}\{U_-, X\}$. 

After collecting the data offline, we calculate an over-approximation of the reachable sets online using Algorithm \ref{alg:LipReachability} \cite{conf:ourjournal}.
We compute a least-squares model $\tilde{M}$ at a linearization point $(u^\star_k,x^\star_k)$ in line \ref{ln:alglipMtilde} where $\mathcal{M}_w=\zono{C_{\mathcal{M}_w},\tilde{G}_{\mathcal{M}_w}}$ is a the noise matrix zonotope \cite{conf:ourjournal} with center matrix $C_{\mathcal{M}_w}$ and a list of generator matrices $\tilde{G}_{\mathcal{M}_w}$. Then, we compute a zonotope that over-approximates the model mismatch and the nonlinearity terms in lines \ref{ln:algliplower} to \ref{ln:alglipZl}.
Given that the data have a limited covering radius, we compute a Lipschitz zonotope in line \ref{ln:alglipZeps} to provide guarantees.
Next, we perform the reachability recursion in line \ref{ln:alglipRhat} given the previously computed zonotopes. Note that the Lipschitz constant $L^\star$ and covering radius $\delta$ can be computed as proposed in \cite{conf:ourjournal,conf:montenbruckLipschitz,conf:novaraLipschitz}.

\begin{algorithm}[t]
  \caption{Reachability analysis for Lipschitz nonlinear system}
  \label{alg:LipReachability}
  \textbf{Input}: input-state trajectories $D = (U_-,X)$, initial set $\mathcal{X}_{0}$, process noise zonotope $\mathcal{Z}_w$ and matrix zonotope $\mathcal{M}_w=\zono{C_{\mathcal{M}_w},\tilde{G}_{\mathcal{M}_w}}$, Lipschitz constant $L^\star$, covering radius $\delta$, input zonotope $\mathcal{U}_k$, data-driven zonotope $\hat{\mathcal{Z}}_{k-1}$\\
  \textbf{Output}: data-driven zonotope $\hat{\mathcal{Z}}_{k}$ 
    \begin{algorithmic}[1]
    
        \State $\tilde{M} = (X_+ - C_{\mathcal{M}_w}) \begin{bmatrix} 
        1_{1 \times T}\\ X_{-}-1 \otimes x^\star_k \\ U_{-}-1 \otimes u^\star_k 
        \end{bmatrix}^\dagger$ \label{ln:alglipMtilde}
        
        \State $\overline{l}  =\max_j \Bigg( {(X_{+})}_{.,j} - \tilde{M} \begin{bmatrix}1\\
        {(X_{-})}_{.,j} -  x^*_k\\ {(U_{-})}_{.,j} - u_k^*\end{bmatrix} \Bigg)$ \label{ln:algliplower}
        
        \State $\underline{l}  =\min_j \Bigg( {(X_{+})}_{.,j} - \tilde{M} \begin{bmatrix}1\\ 
        {(X_{-})}_{.,j} -  x^*_k\\ {(U_{-})}_{.,j} - u_k^*\end{bmatrix} \Bigg)$ \label{ln:alglipupper}
        
        \State $\mathcal{Z}_{L} = \text{zonotope}(\underline{l},\overline{l}) - \mathcal{Z}_w$ \label{ln:alglipZl}
        
        \State $\mathcal{Z}_\epsilon = \zono{0,\textup{diag}(L^\star \delta,\dots,L^\star \delta)}$ \label{ln:alglipZeps}
        
        \State $\hat{\mathcal{Z}}_{k} = \tilde{M} (1 \times \hat{\mathcal{Z}}_{k-1}  \times \mathcal{U}_k ) +  \mathcal{Z}_L + \mathcal{Z}_\epsilon +  \mathcal{Z}_w$ \label{ln:alglipRhat}
    \end{algorithmic}
\end{algorithm}
\subsection{Problem Statement}
Now that we have introduced the necessary preliminaries, we can detail the problem that we aim to solve. 
\begin{problem}
Given the STL side information $\phi_{k}= \phi_{1,k}\wedge \dots \wedge  \phi_{n_{\phi,k},k}$ with $\phi_{i,k}$ of the form~\eqref{2nd}, $i = 1,\dots,n_{\phi,k}$, a historical data set $D=\{U_-, X\}$  collected from an unknown system model, noise zonotope $\mathcal{Z}_w$, and input zonotope $\mathcal{U}_k$, compute the STL reachable set $\bar{\mathcal{R}}_{N}$ at time step $k~=~N$ starting from initial zonotope $\mathcal{X}_0$ that properly over-approximates the set of states $\mathcal{R}_{\phi,N}$ where
\begin{align}
\mathcal{R}_{\phi,N} = \{& x(N) \in \mathbb{R}^n \big| \, \forall k \in \{0,...,N-1\}:  x(k{+}1) {\models}\phi_{k{+}1},\nonumber\\
& x(k+1) {=} f(x(k),u(k))+ w(k), w(k) \in \mathcal{Z}_w,\nonumber\\
&  u(k) \in \mathcal{U}_k, x(0) \in \mathcal{X}_0,x(0) {\models}\phi_{0}  \}.
\end{align}
\end{problem}
The reachable set $\bar{\mathcal{R}}_{N}$ can represented by a zonotope ${\bar{\mathcal{Z}}_{N} \supseteq \mathcal{R}_{\phi,N}}$ or a constrained zonotope $\bar{\mathcal{C}}_{N} \supseteq \mathcal{R}_{\phi,N}$ .

\section{Reachability Analysis Given STL Side Information}\label{sec:approach}




\setlength{\textfloatsep}{0.5cm}
\setlength{\floatsep}{0.5cm}
\begin{algorithm}[t]
  \caption{Reachability analysis under STL side information using zonotopes}
  \label{alg:STLzono}
  \textbf{Input}: data-driven zonotope $\hat{\mathcal{Z}}_{k}=\zono{\hat{c}_{k},\hat{G}_{k}}$, STL side information $\phi_{i,k}$, $\forall i=1,\dots,n_{\phi,k}$\\
    \textbf{Output}: STL zonotope $\bar{\mathcal{Z}}_{k}=\zono{\bar{c}_{k},\bar{G}_{k}}$ 
    \begin{algorithmic}[1]
        \State $\bar{c}_{k}=\hat{c}_{k}, \, \, \bar{G}_{k}=\hat{G}_{k}$        
        \For{$i=1,\dots,n_{\phi,k}$}
        \State Construct $\mathfrak{h}_{i,k}(x)$ from $\phi_{i,k}$\label{ln:construct}
        \State $\lambda^*_{i,k} = \argmin_{\lambda_{i,k}} \norm{\bar{G}_{k}}^2_F$\label{ln:optlambda}
        \If{$\mathfrak{h}_{i,k}(x)$ is linear} 
        \State\!\!\!\!\!\!// $\mathfrak{h}_{i,k}(x)=r_{i,k}-|H_{i,k}x-y_{i,k}|$
        \State\!\!\!\!\!\!$\bar{c}_{k}=  \bar{c}_{k} +  \lambda^*_{i,k}(y_{i,k} - H_{i,k} \bar{c}_{k})$ \label{ln:stripzono_c}
        \State\!\!\!\!\!\!$\bar{G}_{k} = \begin{bmatrix} (I -  \lambda^*_{i,k} H_{i,k} ) \bar{G}_{k} & \lambda^*_{i,k} r_{i,k} \end{bmatrix}$ \label{ln:stripzono_G}
       \ElsIf{$\mathfrak{h}_{i,k}(x)$ is nonlinear} 
        \State\!\!\!\!\!\!// $\mathfrak{h}_{i,k}(x)=r_{i,k}-|h_{i,k}(x)|$
        \State\!\!\!\!\!\!$\bar{c}_{k} {=} \bar{c}_{k} {-} \lambda^*_{i,k} \bigg(\!\!h_{i,k}(x^*_{i,k}) {+}  \frac{\partial h_{i,k}}{\partial x}|_{x^*_{i,k}} (\bar{c}_{k} {-} x^*_{i,k}) {+} c_{L,i,k}\!\! \bigg)$\label{ln:nonstripzono_c}
        \State\!\!\!\!\!\!$\bar{G}_{k}{=}\begin{bmatrix}(I-\! \lambda^*_{i,k}  \frac{\partial h_{i,k}}{\partial x}|_{x^*_{i,k}}) \bar{G}_{k}\!\! \!\!&\lambda^*_{i,k} r_{i,k}\!\!\!\!  &-\lambda^*_{i,k} G_{L,i,k} \end{bmatrix}$\label{ln:nonstripzono_G}
        \EndIf 
        \EndFor
    \end{algorithmic}
\end{algorithm}

In the previous section, we showed how to generate a data-driven reachable set from input-state data. In this section, we show how to incorporate STL formulas in data-driven reachability analysis. Algorithm \ref{alg:STLzono} summarizes our proposed approach using zonotopes. The input to the algorithm is the data-driven zonotope $\hat{\mathcal{Z}}_{k}$ from Algorithm \ref{alg:LipReachability} and STL side information $\phi_{i,k}$, $i=1,\dots,n_{\phi,k}$, of the form~\eqref{2nd}. 
In line \ref{ln:construct}, we construct a predicate function $\mathfrak{h}_{i,k}(x)$ from $\phi_{i,k}$, 
such that if $\mathfrak{h}_{i,k}(x)\geq 0$, then $x\models\phi_{i,k}$ \cite{fainekos2009robustness}. We consider first the linear case, where $\mathfrak{h}_{i,k}(x)$ is a linear formula with respect to $x$. In this case, we represent $\mathfrak{h}_{i,k}(x)$ by a linear strip $\mathcal{S}_{i,k}$ in \eqref{eq:strip} by having $\mathfrak{h}_{i,k}(x)=r_{i,k}-|H_{i,k}x-y_{i,k}|$. The intersection between the linear strip $\mathfrak{h}_{i,k}(x)\geq 0$ and the current zonotope is provided in lines \ref{ln:stripzono_c} and \ref{ln:stripzono_G}. Many scenarios contain nonlinearity in the side information in which we propose to represent the $\mathfrak{h}_{i,k}(x) $ as a nonlinear strip $\mathcal{N}_{i,k}$ and perform an intersection with the data-driven reachable set. More specifically, we consider nonlinear strips in \eqref{df:nonlinstrip} with $\mathfrak{h}_{i,k}(x)=r_{i,k}-|h_{i,k}(x)|$, where the intersection with the reachable zonotope is provided in lines \ref{ln:nonstripzono_c} and \ref{ln:nonstripzono_G}. The optimal parameter $\lambda_{i,k}$ is computed by $\lambda^*_{i,k} {=} \argmin_{\lambda_{i,k}} \norm{\bar{G}_{k}}^2_F$ in line \ref{ln:optlambda}. 





\begin{algorithm}[t]
  \caption{Reachability analysis under STL side information using constrained zonotopes}
  \label{alg:STLconzono}
  \textbf{Input}: data-driven constrained zonotope $\hat{\mathcal{C}}_{k}{=}\zono{\hat{c}_k,\hat{G}_k,\hat{A}_k$, $\hat{b}_k}$, STL side information $\phi_{i,k}$, $i{=}1,{\dots},n_{\phi,k}$\\
    \textbf{Output}: STL constrained zonotope $\bar{\mathcal{C}}_{k}{=}\zono{\bar{c}_k,\bar{G}_k,\bar{A}_k,\bar{b}_k}$ 
    \begin{algorithmic}[1]
        \State $\bar{c}_{k}=\hat{c}_{k}, \, \, \bar{G}_{k}=\hat{G}_{k}, \, \, \bar{A}_{k}=\hat{A}_{k}, \, \, \bar{b}_{k}=\hat{b}_{k}$ 
        \For{$i=1,\dots,n_{\phi,k}$}
        \State Construct $\mathfrak{h}_{i,k}(x)$ from $\phi_{i,k}$\label{ln:constructconzono}
        \If{$\mathfrak{h}_{i,k}(x)$ is linear}
        \State // $\mathfrak{h}_{i,k}(x)=r_{i,k}-|H_{i,k}x-y_{i,k}| $
        \State  $\bar{c}_k = \bar{c}_k, \, \, \bar{G}_k = \bar{G}_k$\label{ln:stripconzono_c}
        \State  $\bar{A}_k {=} \begin{bmatrix}\bar{A}_k &0 \\ H_{i,k} \bar{G}_k &-r_{i,k} \end{bmatrix}, \, \, \bar{b}_k {=} \begin{bmatrix}\bar{b}_k \\ y_{i,k} - H_{i,k} {c}^\prime_k\end{bmatrix}$\label{ln:stripconzono_A}
        \ElsIf{$\mathfrak{h}_{i,k}(x)$ is nonlinear} 
        \State // $\mathfrak{h}_{i,k}(x)=r_{i,k}-|h_{i,k}(x)| $
        \State $\bar{c}_{k} = \bar{c}_{k},\, \, \bar{G}_{k} = \bar{G}_{k}$\label{ln:nonstripconzono_c}
        \State Compute $\mathcal{Z}_{L,i,k}= \zono{c_{L,i,k},G_{L,i,k}}$ \cite[p.65]{conf:thesisalthoff}
  \State $\bar{A}_k = \begin{bmatrix}\bar{A}_k &0& 0 \\ 
  \frac{\partial h_{i,k}}{\partial x}|_{x^*_{i,k}} \bar{G}_{k} &-r_{i,k}& G_{L,i,k}   \end{bmatrix}$\label{ln:nonstripconzono_A}
  \State $\bar{b}_k = \begin{bmatrix}\bar{b}_k \\ -h_{i,k}(x^*_{i,k}) {-} \frac{\partial h_{i,k}}{\partial x}|_{x^*_{i,k}} ( \bar{c}_k {-} x^*_{i,k}) {-} c_{L,i,k} \end{bmatrix}$\label{ln:nonstripconzono_b}
        \EndIf 
        \EndFor
    \end{algorithmic}
\end{algorithm}

The reachable set $\hat{\mathcal{R}}_k$ can be represented by a zonotope $\hat{\mathcal{Z}}_k$ from Algorithm \ref{alg:LipReachability} or as a constrained zonotope $\hat{\mathcal{C}}_k$ \cite{conf:ourjournal}. Using constrained zonotopes allows for less conservative results, but come with extra computational cost. We propose Algorithm \ref{alg:STLconzono} to compute reachable sets under STL side information using constrained zonotope. Similar to Algorithm \ref{alg:STLzono}, we construct $\mathfrak{h}_{i,k}(x)$ from $\phi_{i,k}$ in line \ref{ln:constructconzono}. Then, we provide intersection between constrained zonotope and linear strip $\mathfrak{h}_{i,k}(x) \geq 0$ in lines \ref{ln:stripconzono_c} and \ref{ln:stripconzono_A}. In case of nonlinear $\mathfrak{h}_{i,k}(x)$, we provide the intersection in lines \ref{ln:nonstripconzono_c} to \ref{ln:nonstripconzono_b}. In both Algorithms \ref{alg:STLzono} and \ref{alg:STLconzono}, we guarantee state inclusion by providing an over-approximated intersection between the data-driven reachable set $\hat{\mathcal{R}}_k$ and the $\mathfrak{h}_{i,k}(x)\geq 0$. 
To guarantee state inclusion in the STL generated set in case of nonlinear $\mathfrak{h}_{i,k}(x)$, we linearize and over-approximate the infinite Taylor series by a first order Taylor series and its Lagrange remainder $\mathcal{Z}_{L,i,k}= \zono{c_{L,i,k},G_{L,i,k}}$ \cite[p.65]{conf:thesisalthoff}. The next theorems shows the provided guarantees.

\begin{theorem} \label{th:nonlinzonostrip}
Algorithm \ref{alg:STLzono} provides reachability analysis with state inclusion guarantees under STL side information, i.e., $\bar{\mathcal{Z}}_{k} \supseteq \mathcal{R}_{\phi,k}$. 
\end{theorem}
\begin{proof}
%
%
In order to prove state inclusion guarantees, we show that the resultant intersection $\bar{\mathcal{Z}}_{k}$ between $\mathcal{N}_{i,k}: \mathfrak{h}_{i,k}(x) \geq 0$ and the data-driven zonotope $\hat{\mathcal{Z}}_{k}=\zono{\hat{c}_{k},\hat{G}_{k}}$ contains the state in all cases. We omit the proof in the linear case as it follows immediately from \cite[Prop.1]{conf:stripzono}. We prove the guaranteed intersection in the nonlinear case as follows: We aim to find the zonotope that over-approximates the intersection. Let $x {\in} (\hat{\mathcal{Z}}_{k} {\cap} \mathcal{N}_{i,k} )$, then there is a $ z_k \in [-1_{\gamma_{\hat{\mathcal{Z}}} \times 1},1_{\gamma_{\hat{\mathcal{Z}}} \times 1}]$, where
\begin{align}
    x = \hat{c}_{k} + \hat{G}_{k} z_k. \label{equ:x_zono}
\end{align}
Adding and subtracting $ \lambda_{i,k} \frac{\partial h_{i,k}}{\partial x}|_{x^*_{i,k}} \hat{G}_{k} z_k$ to \eqref{equ:x_zono} results in 
 \begin{align}
    x {=} \hat{c}_{k} {+} \lambda_{i,k}  \frac{\partial h_{i,k}}{\partial x}|_{x^*_{i,k}} \hat{G}_{k} z_k {+} \bigg( I {-}  \lambda_{i,k}  \frac{\partial h_{i,k}}{\partial x}|_{x^*_{i,k}}\bigg) \hat{G}_{k} z_k.\label{equ:x_zono2}
 \end{align}
 Given that $x \in (\hat{\mathcal{Z}}_{k} \cap \mathcal{N}_{i,k} )$, then $x \in \mathcal{N}_{i,k}$, i.e., there exists $d \in [-1_{p \times 1},1_{p \times 1}]$ for $\mathcal{N}_{i,k}$ such that:
 \begin{align}
 h_{i,k}(x^*_{i,k}) &+  \frac{\partial h_{i,k}}{\partial x}|_{x^*_{i,k}} (x - x^*_{i,k}) + \dots  =  r_{i,k}d. 
 \label{equ:bj1}
 \end{align}
 Inserting \eqref{equ:x_zono} into \eqref{equ:bj1}
 \begin{align*}
 -  h_{i,k}(x^*_{i,k}) -  \frac{\partial h_{i,k}}{\partial x}|_{x^*_{i,k}} (\hat{c}_{k} - x^*_{i,k}) - \dots & +  r_{i,k}d =\nonumber\\
 &\frac{\partial h_{i,k}}{\partial x}|_{x^*_{i,k}} \hat{G}_{k} z_k   
 \end{align*}
 Considering the Lagrange remainder $\mathcal{Z}_{L,i,k}{=} \zono{c_{L,i,k},G_{L,i,k}}$ \cite[p.65]{conf:thesisalthoff} results in
 \begin{align}
   \frac{\partial h_{i,k}}{\partial x}|_{x^*_{i,k}} \hat{G}_{k} z_k  \in & - h_{i,k}(x^*_{i,k}) {-}  \frac{\partial h_{i,k}}{\partial x}|_{x^*_{i,k}} (\hat{c}_{k} {-} x^*_{i,k}) \nonumber \\  
 & {-} \mathcal{Z}_{L,i,k} {+}  r_{i,k}d. 
   \label{equ:x_bj}
 \end{align}
Inserting \eqref{equ:x_bj} in \eqref{equ:x_zono2} results in 
\begin{align*}
    x\in& \hat{c}_{k} {+}  \lambda_{i,k} \bigg( {-} h_{i,k}(x^*_{i,k}) {-}  \frac{\partial h_{i,k}}{\partial x}|_{x^*_{i,k}} (\hat{c}_{k} - x^*_{i,k}) - \mathcal{Z}_{L,i,k} \nonumber\\
    &{+}  r_{i,k}d\bigg)  + \bigg(I -  \lambda_{i,k}  \frac{\partial h_{i,k}}{\partial x}|_{x^*_{i,k}}\bigg) \hat{G}_{k} z_k \nonumber\\
{=}& \underbrace{\hat{c}_{k} {-} \lambda_{i,k} \bigg( h_{i,k}(x^*_{i,k})  {+} \frac{\partial h_{i,k}}{\partial x}|_{x^*_{i,k}} (\hat{c}_{k} {-} x^*_{i,k}) {+} c_{L,i,k} \bigg)}_{\bar{c}_{k}} \nonumber\\
& {+} \underbrace{\begin{bmatrix}(I{-} \lambda_{i,k}  \frac{\partial h_{i,k}}{\partial x}|_{x^*_{i,k}}) \hat{G}_{k}&\!\! \lambda_{i,k} r_{i,k}&\!\!-\lambda_{i,k} G_{L,i,k} \end{bmatrix}}_{\bar{G}_{k}}\!\!\underbrace{\begin{bmatrix} z_k \\ d \\ z_L 
\end{bmatrix}}_{z_b} 
\end{align*}
Note that $z_b \in [-1_{\gamma_{\bar{\mathcal{Z}}} \times 1},1_{\gamma_{\bar{\mathcal{Z}}} \times 1}]$ as $d \in [-1_{p \times 1},1_{p \times 1}],z_k \in [-1_{\gamma_{\hat{\mathcal{Z}}} \times 1},1_{\gamma_{\hat{\mathcal{Z}}} \times 1}]$, and $z_L \in [-1_{\gamma_{\mathcal{Z}_L} \times 1},1_{\gamma_{\mathcal{Z}_L} \times 1}]$. Thus, the center and the generator of the over-approximating zonotope are $\bar{c}_{k}$ and $\bar{G}_{k}$, respectively. 
\end{proof}



 \begin{theorem}
 \label{th:nonlinconzonostrip}
 Algorithm \ref{alg:STLconzono} provides reachability analysis with state inclusion guarantees under STL side information, i.e., $\bar{\mathcal{C}}_{k} \supseteq \mathcal{R}_{\phi,k}$.  
%
 \end{theorem}
 \begin{proof}
Similar to the proof of Theorem \ref{th:nonlinzonostrip}, we omit the proof for the linear case as it follows immediately from \cite[Prop.1]{conf:const_zono} and prove the guaranteed intersection in the nonlinear case as follows: Let $x \in (\hat{\mathcal{C}}_k \cap \mathcal{N}_{i,k})$, then there is a $z_k\in\left[-1_{\gamma_{\hat{\mathcal{C}}} \times 1},1_{\gamma_{\hat{\mathcal{C}}} \times 1}\right]$ such that
\begin{align}
   x &= \hat{c}_k + \hat{G}_k z_k, \label{equ:x_zono3}\\
    \hat{A}_k z_k &= \hat{b}_k. \label{equ:z_zono}
\end{align}
%
 Given that $x$ is inside the intersection of the constrained zonotope $\hat{\mathcal{C}}_k$ and $\mathcal{N}_{i,k}$, there exists a $d \in [-1_{p \times 1},1_{p \times 1}]$ such that
 \begin{equation}
  h_{i,k}(x^*_{i,k}) + \frac{\partial h_{i,k}}{\partial x}|_{x^*_{i,k}}\big(x- x^*_{i,k}\big) + \dots  =  r_{i,k}\ d. 
 \label{equ:bj}
 \end{equation}
Inserting \eqref{equ:x_zono3} into \eqref{equ:bj} results in 
 \begin{equation}
  h_{i,k}(x^*_{i,k}) + \frac{\partial h_{i,k}}{\partial x}|_{x^*_{i,k}}\big(\hat{c}_k + \hat{G}_k z_k- x^*_{i,k}\big) + \dots  =  r_{i,k}\ d. 
   \label{equ:z_bj}
 \end{equation}
 We combine \eqref{equ:z_bj} and \eqref{equ:z_zono} while considering the Lagrange remainder yielding
 \begin{align*}
 \underbrace{\begin{bmatrix}\hat{A}_k &0& 0 \\ 
  \frac{\partial h_{i,k}}{\partial x}|_{x^*_{i,k}} \hat{G}_{k} &-r_{i,k}& G_{L,i,k}   \end{bmatrix}}_{\bar{A}_k} &\underbrace{\begin{bmatrix} z_k \\ d \\ z_L \end{bmatrix}}_{z_b} =\nonumber\\ 
 &\!\!\!\!\!\!\!\!\!\!\!\!\!\!\!\!\!\!\!\!\!\!\!\!\!\!\!\!\!\!\!\!\!\!\!\!\!\!\!\!\!\!\!\!\!\!\! \underbrace{\begin{bmatrix}\hat{b}_k \\ -h_{i,k}(x^*_{i,k}) - \frac{\partial h_{i,k}}{\partial x}|_{x^*_{i,k}} ( \hat{c}_k - x^*_{i,k}) - c_{L,i,k} \end{bmatrix}}_{\bar{b}_k}.
 \end{align*}
 Note that we consider the superset consisting the equality~\eqref{equ:z_bj} by solving it for all $d \in [-1_{p \times 1},1_{p \times 1}]$. Then, we can assure that~\eqref{equ:z_bj} is also satisfied.
\end{proof}

In the next section, we empirically show that the reachable sets computed from these intersections is a practical improvement compared to original data-driven reachable sets.



\section{Evaluation}\label{sec:eval}

\begin{figure}[t]
    \includegraphics[width=0.48\textwidth]{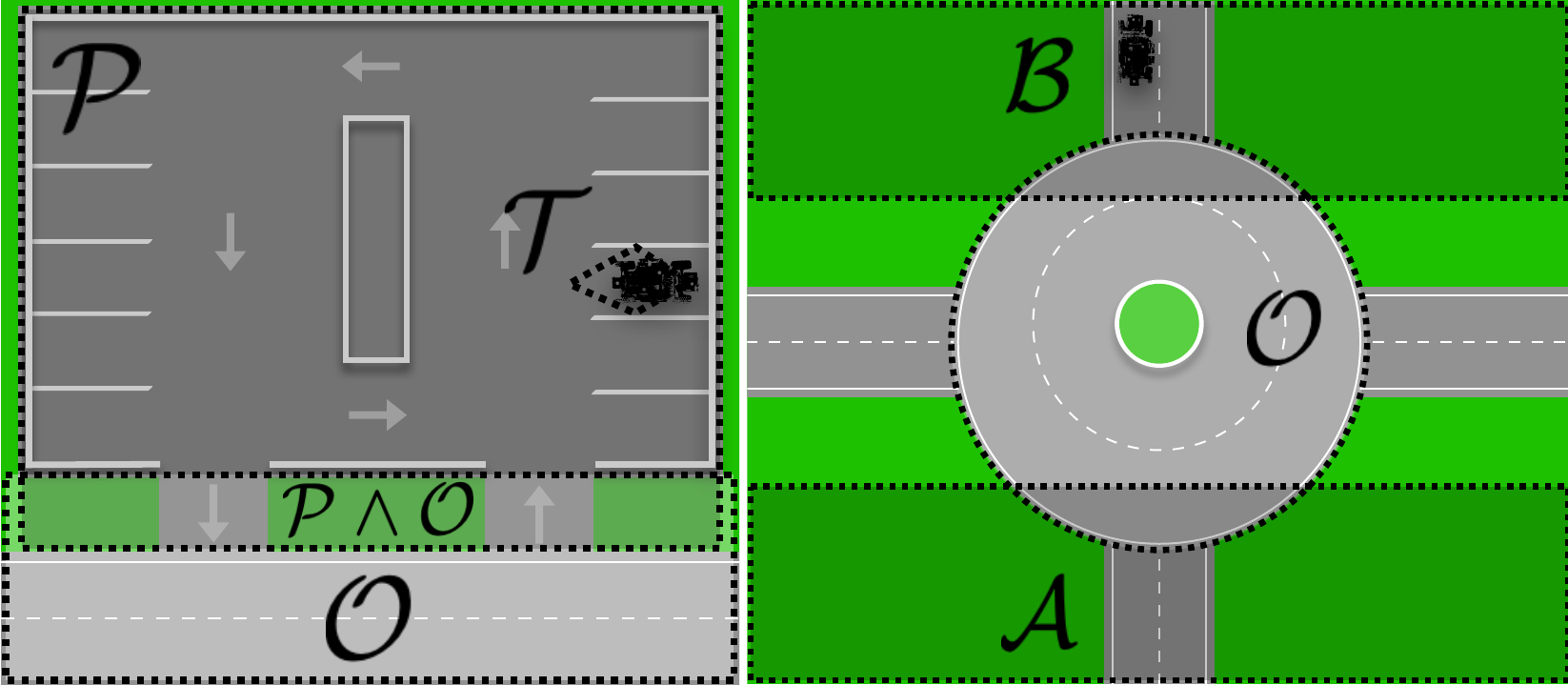}
    \caption{Illustration of our parking lot example on the left and roundabout example on the right.}
    \label{fig:parkround}
    \vspace{-4mm}
\end{figure}

\begin{table}[t]
\caption{Average volumes in the parking example.}
\label{tab:parking}
\centering
\normalsize
\begin{tabular}{c c c}
\toprule
  & Zonotope & Constrained zonotope \\
%
\midrule
   No constraints      & 9.722 & - \\
 $\phi_p$ constraints   &  9.311 &7.042 \\
 $\phi_\theta$ constraints  &0.124& 0.076 \\
\bottomrule
\end{tabular}
\end{table}

In this section, we detail the application of our method to two examples. 
Readers can find an overview video of our experiments conducted at the \href{https://www.kth.se/is/dcs/research/control-of-transport/smart-mobility-lab/smart-mobility-lab-1.441539}{Smart Mobility Lab} at \href{https://bit.ly/DataReachSTL}{[https://bit.ly/DataReachSTL]}. 

 For our experimental platform, we represent a vehicle $V$ with a SVEA vehicle~\cite{Jiang2020a}. We use historical data sets of length $1000$ points gathered from the same car from other driving scenarios than the presented ones. We perform a single-step reachability analysis for each example, and we manually operate the car such that its behavior satisfies the known side information. Measurements for both the historical data sets and our two examples are made using a motion capture system. The assumed 
 process noise zonotope is $\mathcal{Z}_w =\zono{0,\begin{bmatrix} 0.9 & 0.9\end{bmatrix}^T}$ and measurement noise zonotope of value $\zono{0,\begin{bmatrix} 0.01 & 0.01\end{bmatrix}^T}$. 
 For both examples, let $V$ and its environment be defined over $\mathbb{R}^2$. In other words, $V$'s state $x\in\mathbb{R}^2$ is written as $x=[x_1 \ \, x_2]^\top$, where $x_1$ and $x_2$ are the x and y positions of $V$. Now, in the following sections, we will introduce our two scenarios for $V$ and present the results for each case. 

\begin{figure}[t]
    \centering
    \includegraphics[width=0.27\textwidth]{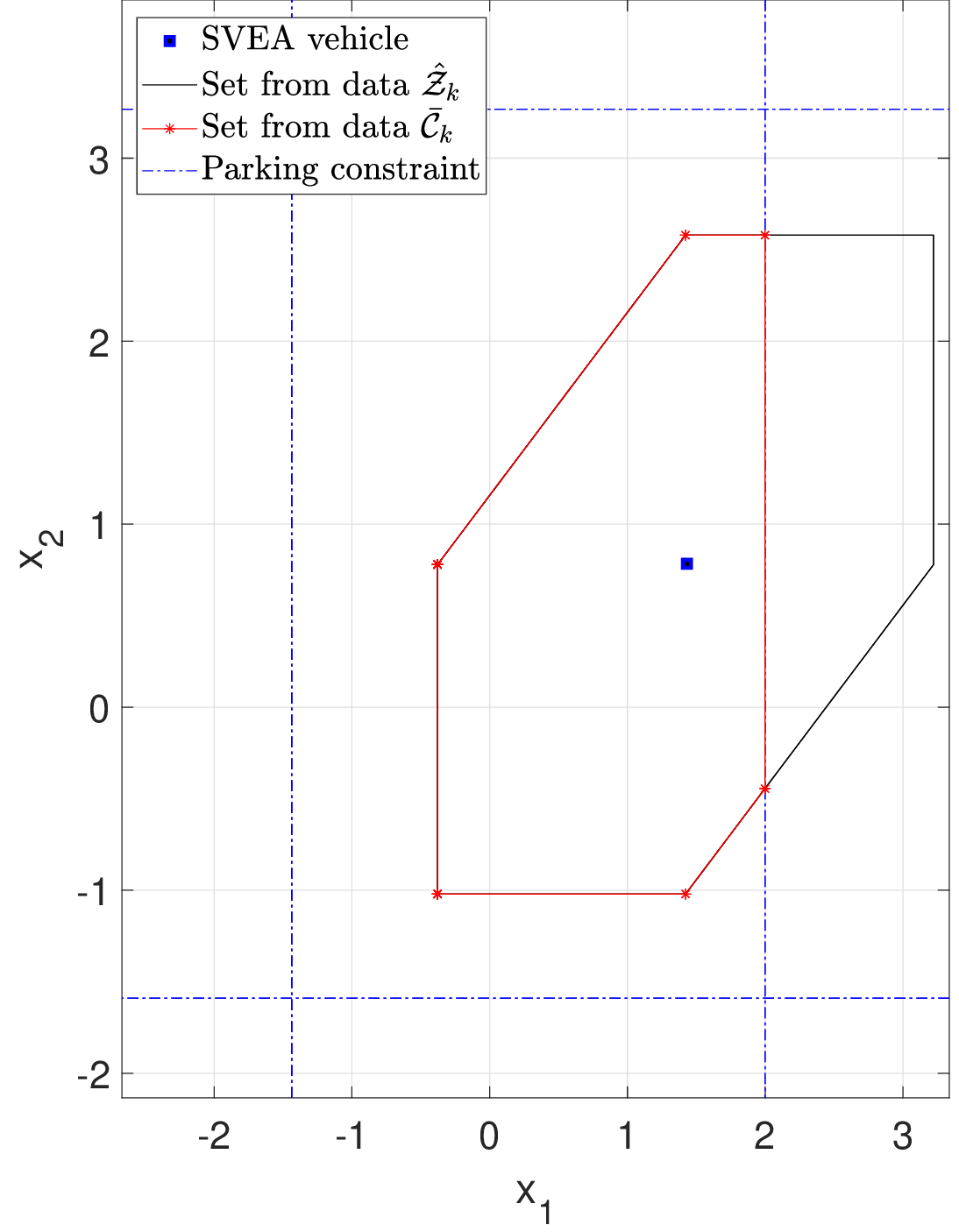}
    \caption{Snapshot showing the result of constraining a reachable set from the parking example.}
    \label{fig:snapshot}
    \vspace{-4mm}
\end{figure}

\begin{figure*}[!htbp]
\vspace{-0.05cm}
    \centering
    \begin{tabular}{ p{0.330\textwidth}  p{0.330\textwidth} p{0.330\textwidth}}
    \resizebox{0.32\textwidth}{!}{
            \begin{subfigure}[h]{0.34\textwidth}
     \centering
        \includegraphics[scale=0.22]{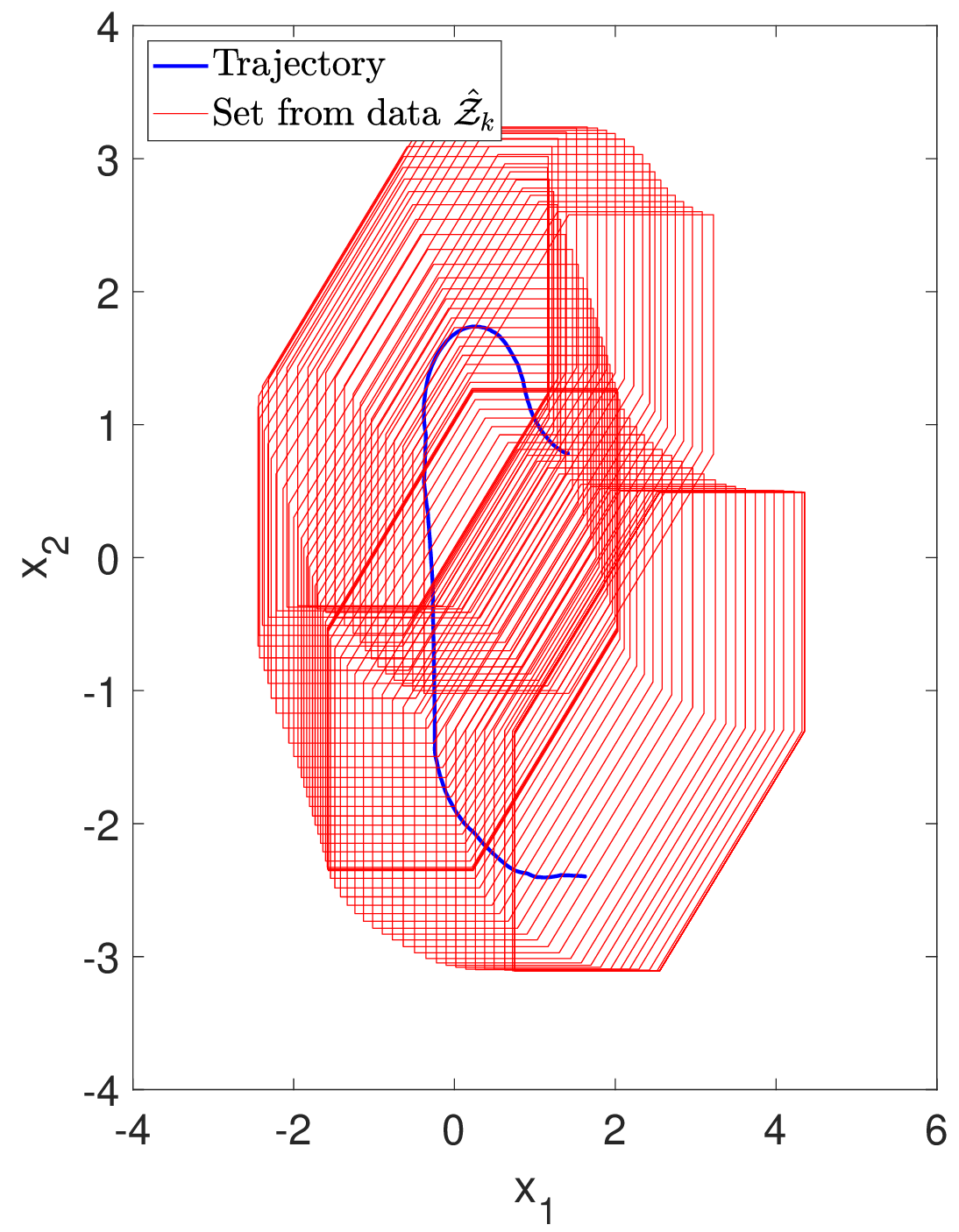}
        \caption{}
        \label{fig:data_driven_zono}
    \end{subfigure}
      } 
  &
  \resizebox{0.32\textwidth}{!}{
            \begin{subfigure}[h]{0.34\textwidth}
     \centering
        \includegraphics[scale=0.22]{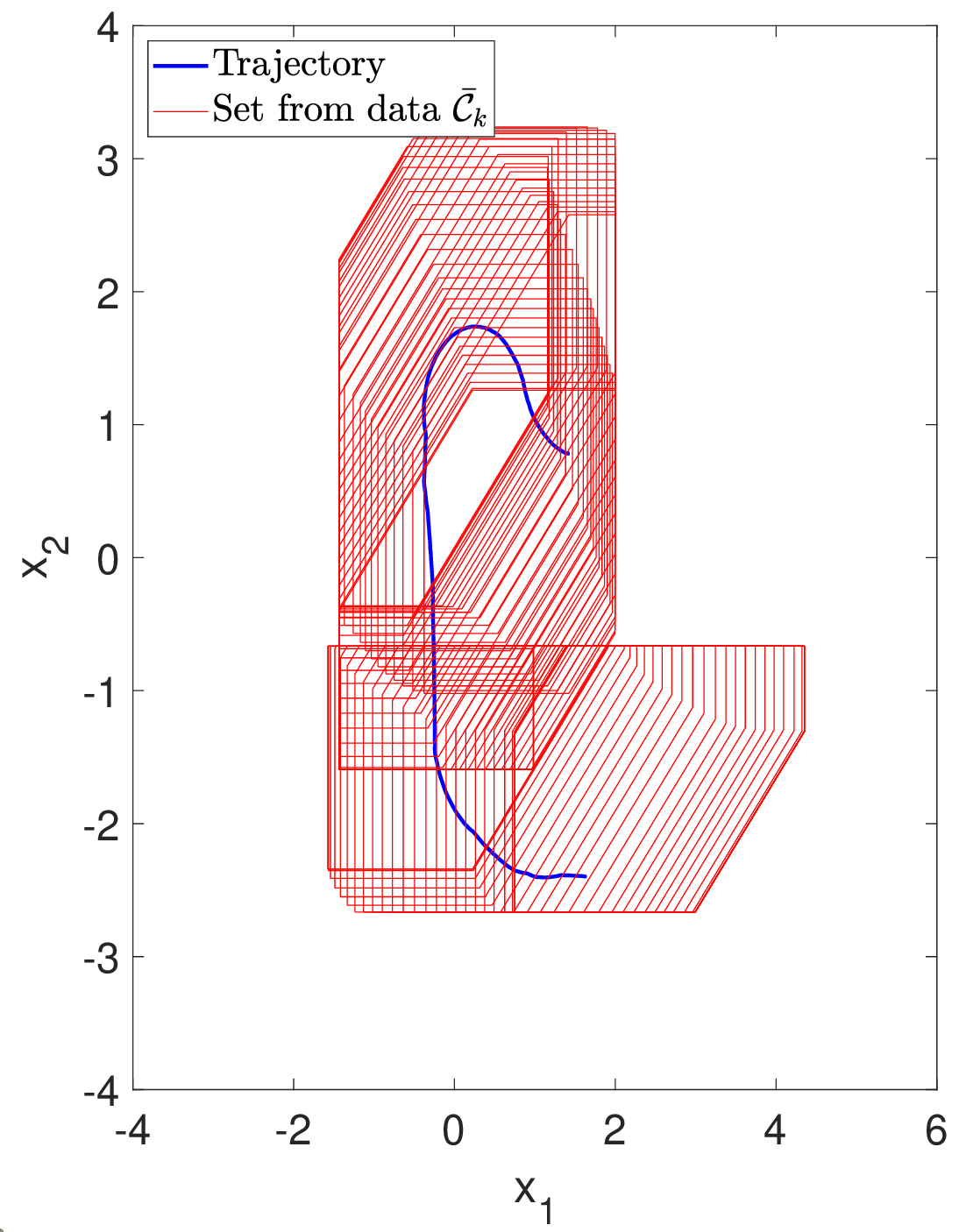}
        \caption{}
        \label{fig:parkingconst_cmz}
    \end{subfigure}
     }
        &
  \resizebox{0.32\textwidth}{!}{
            \begin{subfigure}[h]{0.34\textwidth}
     \centering
        \includegraphics[scale=0.22]{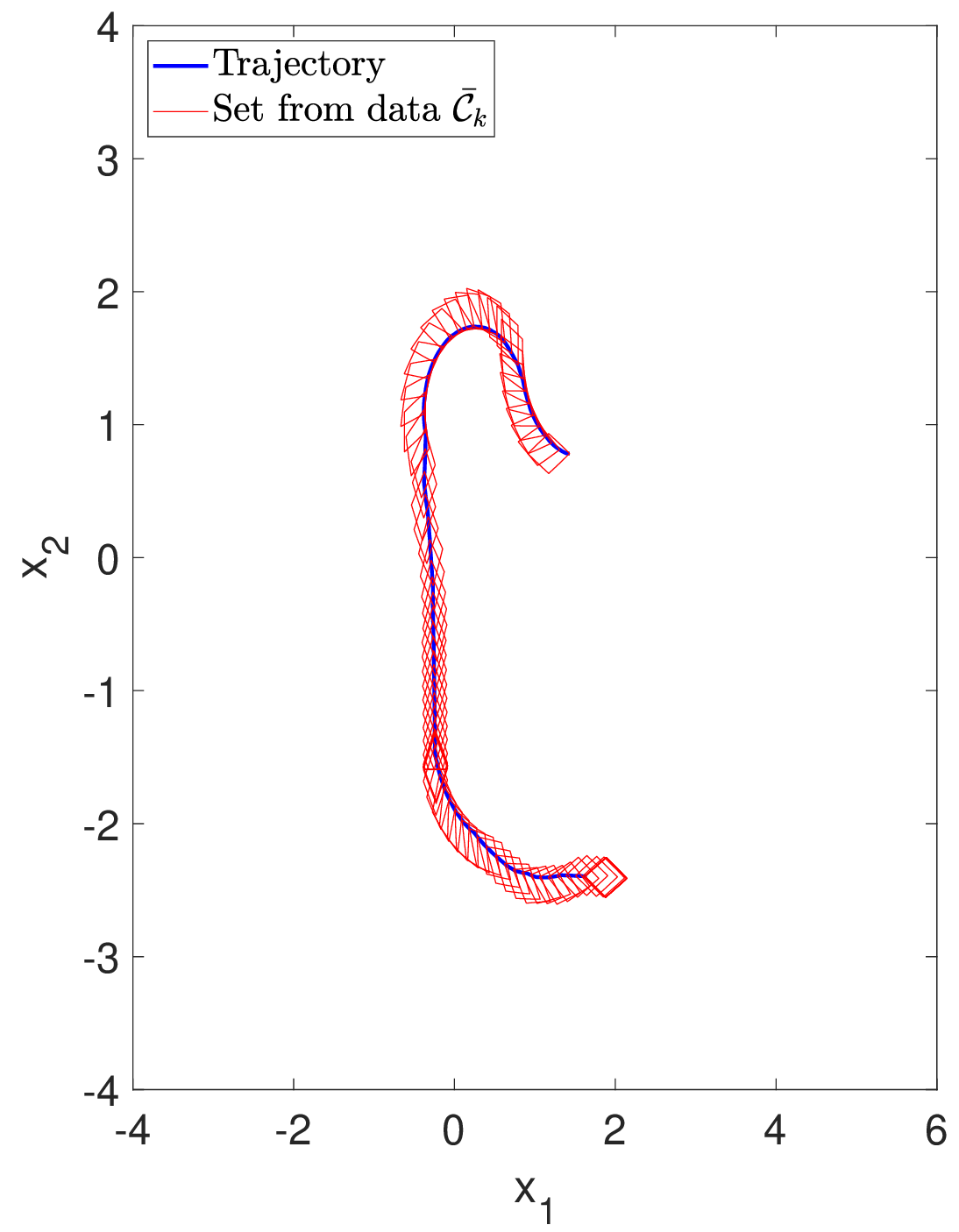}
        \caption{}
        \label{fig:theta_cmz_scalled}
    \end{subfigure}
     }
 \end{tabular}
\caption{The reachable sets without constraint in (a), with $\phi_p$ constraints in (b) and $\phi_\theta$ constraint in (c) for the parking lot example.}
    \label{fig:prakingreach}
    \vspace{-4mm}
\end{figure*}

\begin{figure*}[!htbp]
\vspace{-0.05cm}
    \centering
    \begin{tabular}{ p{0.330\textwidth}  p{0.330\textwidth} p{0.330\textwidth}}
    \resizebox{0.32\textwidth}{!}{
            \begin{subfigure}[h]{0.34\textwidth}
     \centering
        \includegraphics[scale=0.22]{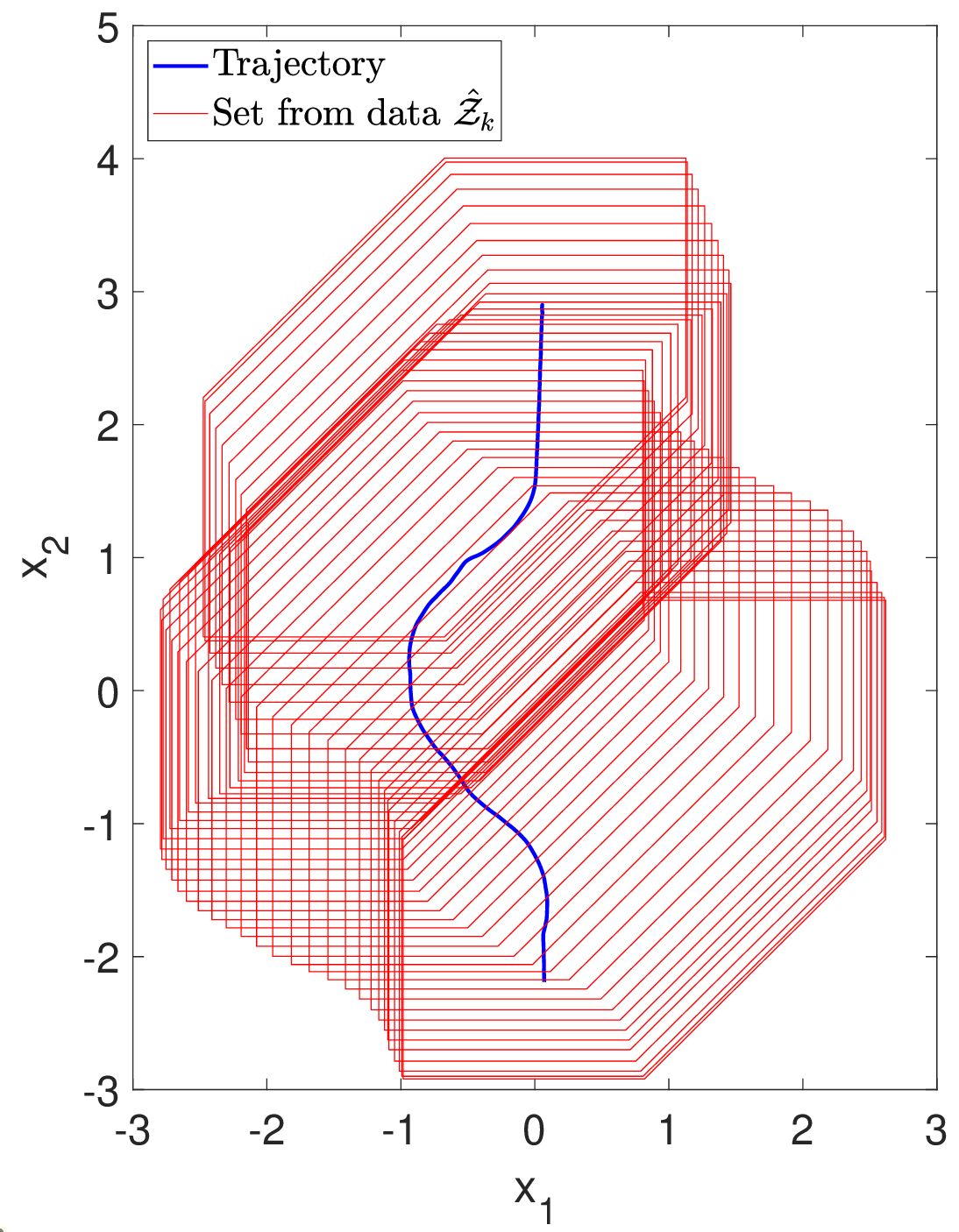}
        \caption{}
        \label{fig:rounddatadriven}
    \end{subfigure}
      } 
  &
  \resizebox{0.32\textwidth}{!}{
            \begin{subfigure}[h]{0.34\textwidth}
     \centering
        \includegraphics[scale=0.22]{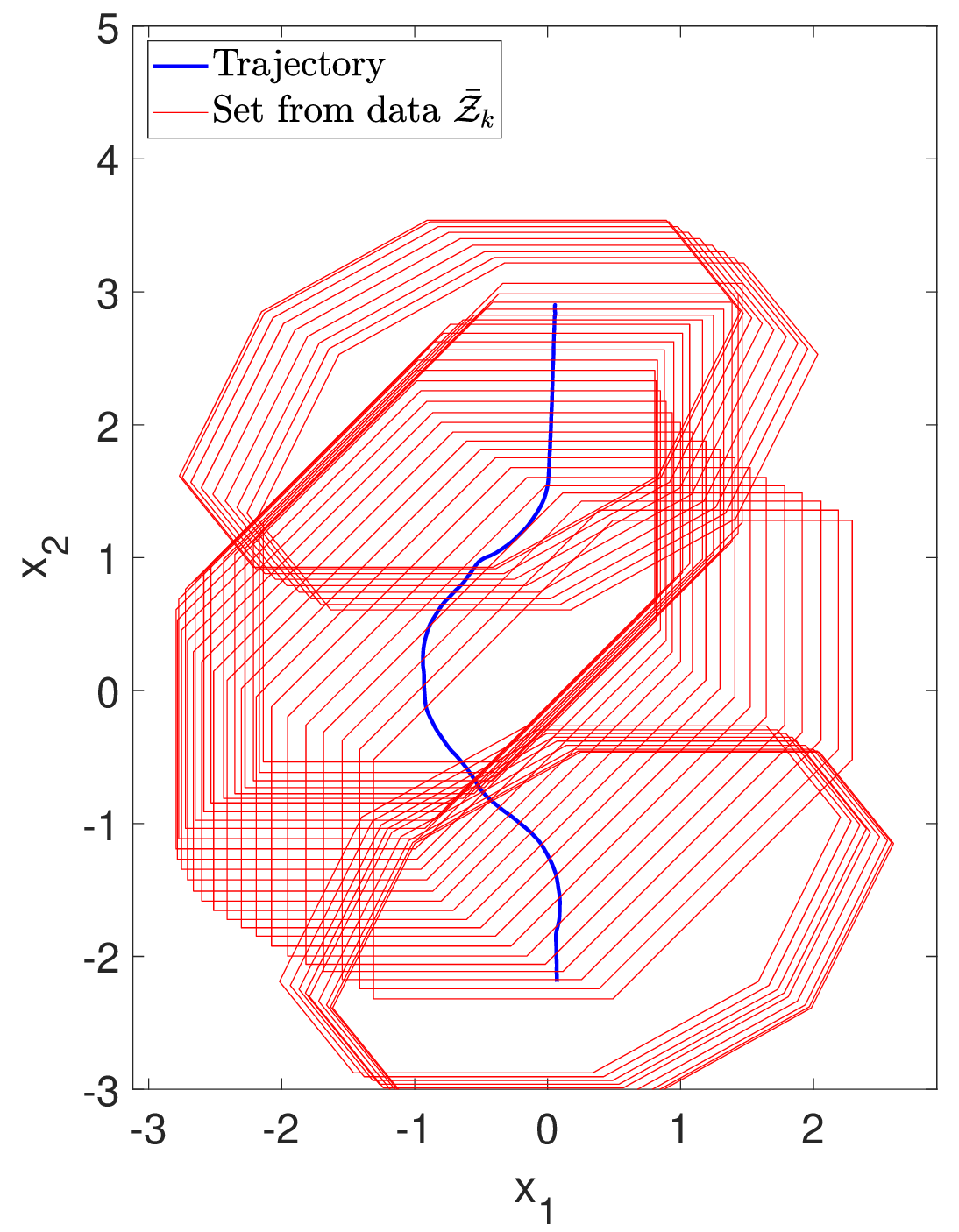}
        \caption{}
        \label{fig:roundzono}
    \end{subfigure}
     }
        &
  \resizebox{0.32\textwidth}{!}{
            \begin{subfigure}[h]{0.34\textwidth}
     \centering
        \includegraphics[scale=0.22]{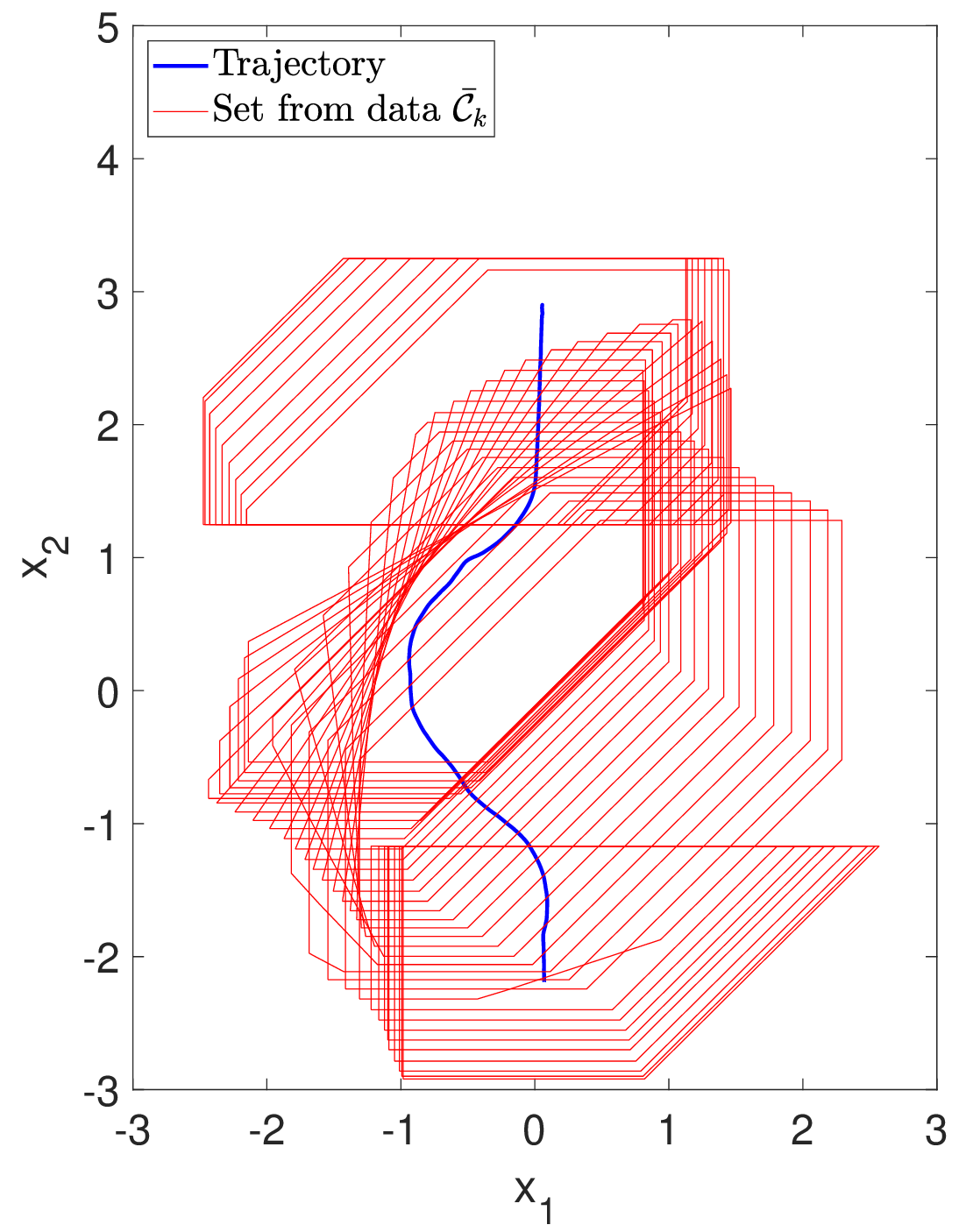}
        \caption{}
        \label{fig:roundcmz}
    \end{subfigure}
     }
 \end{tabular}
\caption{The reachable sets without constraint in (a), with $\phi_r$ constraints using zonotope and  constrained zonotope in (b) and (c), respectively, for the roundabout example.}
    \label{fig:noisycontreach}
    \vspace{-5mm}
\end{figure*}

\begin{table}[t]
\caption{Average volumes in the roundabout example.}
\label{tab:round}
\centering
\normalsize
\begin{tabular}{c c c}
\toprule
  & Zonotope & Constrained zonotope\\
%
\midrule
   No constraints     &  9.722 & - \\
 $\phi_r$ constraints   & 9.109  &5.956\\
\bottomrule
\end{tabular}
\end{table}

\subsection{Parking Lot Example}
 In this example, we consider side information that contains only linear spatial constraints. Suppose $V$ is parked in the parking lot and is scheduled to depart the parking lot soon. As denoted in Fig.~\ref{fig:parkround}, let the set of states corresponding to the parking region be $\mathcal{P}\subset\mathbb{R}^2$ and the set of states corresponding to the outside of the parking region (the street) be $\mathcal{O}\subset\mathbb{R}^2$. Note, the entrance and exit of the parking lot is considered both part of the parking region and the street. We know that $V$ is scheduled to leave the parking region within 25 seconds of the start of our scenario. Thus, we can write the following STL formula as the known side information about $V$:
%
 %
%
\begin{align}\label{eq:plot_constraint}
\phi_p::{=} G_{[0, 25]}(\mathcal{P}) \ {\wedge} \ F_{[0, 25]}(\mathcal{P}\wedge\mathcal{O}) \ {\wedge} \ G_{[25, 40]}(\mathcal{O}).
\end{align}
%
%
%
%
We can find the functions $\mathfrak{h}_1$ to $\mathfrak{h}_5$, which encode \eqref{eq:plot_constraint}:
\begin{align*}
    & \mathfrak{h}_1(x_1,x_2) =  1.7175 - |x_1 - 0.2805|, t\in[0,25],\\
    & \mathfrak{h}_2(x_1,x_2) =  2.429 - |x_2 - 0.839|, t\in[0,25],\\
   & \mathfrak{h}_3(x_1,x_2) =  1.3045-|x_1 + 0.3225|, t\in[24,25],\\
    & \mathfrak{h}_4(x_1,x_2) =  0.453-|x_2 + 1.137|, t\in[24,25],\\
    & \mathfrak{h}_5(x_1,x_2) =  1 -|x_2 + 1.665|, t\in[25,40],
\end{align*}
where $\mathfrak{h}_1$ and $\mathfrak{h}_2$ models our knowledge of $V$'s time within the region $\mathcal{P}$, $\mathfrak{h}_3$ and $\mathfrak{h}_4$ encodes $V$ eventually reaching the exit region $\mathcal{P}\wedge\mathcal{O}$ before $t=25$, and $\mathfrak{h}_5$ corresponds to our knowledge of when $V$ departs to $\mathcal{O}$. Fig. \ref{fig:snapshot} shows a snapshot of the data-driven reachable sets before and after being constrained by $\phi_p$ at $t=1$. We show the unconstrained, data-driven reachable sets in Fig.~\ref{fig:data_driven_zono} and the STL reachable sets constrained by $\phi_p$ in Fig.~\ref{fig:parkingconst_cmz}. 

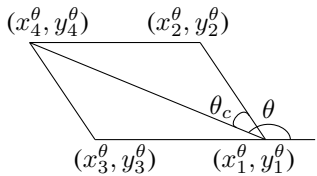
\begin{figure}[h!]
    \centering
\begin{tikzpicture}[x=0.75pt,y=0.75pt,yscale=-0.5,xscale=0.5]

\draw    (99,31.5) -- (271,31.5) ;
\draw    (271,31.5) -- (337,129.5) ;
\draw    (99,31.5) -- (165,129.5) ;
\draw    (99,31.5) -- (337,129.5) ;
\draw    (337,129.5) -- (387,129.5) ;
\draw    (320.2,122.2) .. controls (336.6,106.2) and (362,116.5) .. (362,129.5) ;
\draw    (308,117.25) .. controls (297,115) and (314,93.25) .. (320.5,105.5) ;
\draw    (165,129.5) -- (337,129.5) ;

\draw (356.4,98.9) node   [align=left] {\begin{minipage}[lt]{17.14pt}\setlength\topsep{0pt}
{\normalsize $\theta$}
\end{minipage}};
\draw (316.5,96.7) node   [align=left] {\begin{minipage}[lt]{27.06pt}\setlength\topsep{0pt}
{\normalsize $\theta_{c}$}
\end{minipage}};
\draw (327.2,148.85) node   [align=left] {\begin{minipage}[lt]{32.64pt}\setlength\topsep{0pt}
{\normalsize ($x_1^\theta,y_1^\theta$)}
\end{minipage}};
\draw (261,10.45) node   [align=left] {\begin{minipage}[lt]{32.64pt}\setlength\topsep{0pt}
{\normalsize ($x_2^\theta,y_2^\theta$)}
\end{minipage}};
\draw (119.4,10.45) node   [align=left] {\begin{minipage}[lt]{32.64pt}\setlength\topsep{0pt}
{\normalsize ($x_4^\theta,y_4^\theta$)}
\end{minipage}};
\draw (187.2,148.85) node   [align=left] {\begin{minipage}[lt]{32.64pt}\setlength\topsep{0pt}
{\normalsize ($x_3^\theta,y_3^\theta$)}
\end{minipage}};
\draw (146.6,79.4) node   [align=left] {\begin{minipage}[lt]{11.42pt}\setlength\topsep{0pt}
{\normalsize }
\end{minipage}};
\draw (244.6,119.45) node   [align=left] {\begin{minipage}[lt]{8.67pt}\setlength\topsep{0pt}
{\normalsize }
\end{minipage}};
\end{tikzpicture}
    \caption{Constrained region for $\mathcal{T}(x)$.}
    \label{rec}
    \vspace{-5mm}
\end{figure}

Then, suppose we know the upper limit of $V$'s capability to move forward and change heading between each sampling time. Let this set be denoted by $\mathcal{T}(x)$. Then, we can expand~\eqref{eq:plot_constraint} into the following STL formula as the known side information about $V$: $\phi_\theta::=    G_{[0, 40]}(\mathcal{T}(x)) \ \wedge \ G_{[0, 25]}(\mathcal{P}) \ \wedge \ F_{[0, 25]}(\mathcal{P}\wedge\mathcal{O}) \ \wedge \ G_{[25, 40]}(\mathcal{O})$. Now, we find the additional functions $\mathfrak{h}_6$, $\mathfrak{h}_7$, which encode the constraints corresponding to $G_{[0, 40]}(\mathcal{T}(x))$. Let $\theta$ be the heading angle and $\theta_c$ be the known, maximum heading angle change between each sampling time. We derive the constrained rectangular region $\mathcal{T}(x)$, shown in Fig.~\ref{rec}, with the following equations using the edges coordinates $x_{i}^\theta,y_{i}^\theta$, $i=1,\cdots,4$:
\begin{align*}
 &\mathfrak{h}_6(x_1,x_2)=0.5|c_2-c_3|-|-m_2x_1 + x_2 - 0.5(c_1+c_4)|,\\ 
 &\mathfrak{h}_7(x_1,x_2)=0.5|c_1-c_4|-|-m_1x_1 + x_2 - 0.5(c_2+c_3)|, 
\end{align*}
where $m_i=\frac{y_{i+1}^\theta-y_1^\theta}{x_{i+1}^\theta-x_1^\theta}, c_i=-m_ix_1^\theta+y_1^\theta$ for $i=1,2$, $c_3 = -m_2x_2^\theta +y_2^\theta$, and $c_4 = -m_1x_3^\theta +y_3^\theta$. Both $\mathfrak{h}_4$ and $\mathfrak{h}_5$ are defined for $t\in[0, 40]$. The reachable sets using $\phi_\theta$ as side information and constrained zonotope are shown in Fig. \ref{fig:theta_cmz_scalled}. The average volumes of the reachable sets are presented in Table \ref{tab:parking}.

\subsection{Roundabout Example}
We evaluate how the STL-based side information constrains the reachable sets when a nonlinear spatial constraint is included in the side-information. Suppose $V$ enters, drives around, and exits a roundabout intersection. For this example, we assume we have a rough prediction of when $V$ will enter and exit the roundabout. As illustrated in Fig.~\ref{fig:parkround}, let the region before the roundabout be $\mathcal{B}\subset\mathbb{R}^2$, the roundabout itself be $\mathcal{O}\subset\mathbb{R}^2$, and the region after the roundabout be $\mathcal{A}\subset\mathbb{R}^2$. We model the roundabout as a circle and we will use $\mathcal{O}$ to introduce nonlinearity into our side information. Finally, we know that $V$ will enter the roundabout within 4 seconds and will leave the roundabout within $10$ seconds of the start of the scenario. We formalize the side information  with the following STL formula:
%
%
%
\begin{equation}\label{eq:round_constraint}
\phi_r::= G_{[0, 4]}(\mathcal{B}) \ \wedge \ G_{[4, 10]}(\mathcal{O}) \wedge \ G_{[10, 14]} (\mathcal{A}).
\end{equation}
%
%
Accordingly, the functions $\mathfrak{h}_1,\cdots,\mathfrak{h}_3$ encode~\eqref{eq:round_constraint}:
\begin{align*}
    \mathfrak{h}_1(x_1,x_2) &{=}  1 - |x_2 - 2.25|, t\in[0,4],\\
    \mathfrak{h}_2(x_1,x_2) &{=} 1.429 {-} \|[x_1, x_2]^\top {-} [0.307, 0.044]^\top\|, t\in[4, 10], \\
    \mathfrak{h}_3(x_1,x_2) &{=} 1 - |x_2 + 2.169| , t\in[10, 14],
\end{align*}
where $h_1$, $h_2$ and $h_3$ models the satisfaction of the formulae corresponding to the regions $\mathcal{B}$, $\mathcal{O}$ and $\mathcal{A}$, respectively. We show the unconstrained, data-driven reachable sets in Fig.~\ref{fig:rounddatadriven}, the STL reachable sets constrained by  $\phi_r$ using zonotopes in Fig.~\ref{fig:roundzono}, and the STL reachable sets using constrained zonotopes in Fig.~\ref{fig:roundcmz}. The average volumes of the reachable sets are presented in Table \ref{tab:round}.

\vspace{-1mm}
\section{Conclusion}\label{sec:con}
\vspace{-1mm}
We have provided an approach  to achieve less conservative, data-driven reachable sets. We have shown that known, STL-based side information can be used to constrain reachable zonotopes post-analysis, while still maintaining safety guarantees on the resulting constrained zonotopes. 
In future work, we will evaluate our approach on more complex scenarios and potentially apply the work to multi-agent tasks. 

\balance
\bibliographystyle{ieeetr}
\bibliography{ref} 

\end{document}